\documentclass{article}
\PassOptionsToPackage{numbers, compress}{natbib}

\usepackage[main,final]{neurips_2026}

\usepackage[utf8]{inputenc} 
\usepackage[T1]{fontenc}    
\usepackage{hyperref}       
\usepackage{url}            
\usepackage{booktabs}       
\usepackage{amsfonts}       
\usepackage{nicefrac}       
\usepackage{microtype}      
\usepackage{xcolor}         
\usepackage{caption}
\usepackage{amsthm}
\usepackage{multirow}

\newtheorem{theorem}{Theorem}
\newtheorem{assumption}{Assumption}

\newtheorem{corollary}{Corollary}
\newtheorem{definition}{Definition}

\usepackage{subcaption}
\usepackage{graphicx}

\usepackage{amssymb}
\usepackage[mathscr]{euscript}

\usepackage{amsfonts,amsmath,amssymb}
\usepackage{bbm}
\usepackage{xspace}

\usepackage{multicol}
\usepackage{enumitem}

\usepackage[normalem]{ulem}

\title{Black-box model classification under the discriminative factorization}

\author{%
Hayden Helm\thanks{Corresponding author} \\
  Helivan \\
  \texttt{hayden@helivan.io} \\
  \And
  Merrick Ohata \\
  JHU \\
  \And
  Carey E. Priebe \\
  JHU
}

\begin{document}

\maketitle
\setcounter{footnote}{0}

\begin{abstract}
Access to modern generative systems is often restricted to querying an API (the ``black-box'' setting) and many properties of the system are unknown to the user at inference time. 
Recent work has shown that low-dimensional representations of models based on the relationship of their embedded responses to a set of queries are useful for inferring model-level properties. However, the quality of these representations is highly sensitive to the query set. 
We introduce the \emph{discriminative factorization}
to distinguish between high- and low-quality query sets in the context of black-box model-level classification. 
Under this framework, the probability of chance-level classification is bounded by a term that decays exponentially in the query budget.
On three auditing tasks, estimated factorization parameters predict the empirical performance decay rate. 
We conclude by showing that query sets selected using the estimated factorization reproduce the empirical ordering of oracle query sets.
\end{abstract}

\section{Introduction}
As model providers increasingly restrict access to only API endpoints 
\citep{openai2025gpt5, anthropic2025sonnet45, 
googledeepmind2025gemini25}, black-box interactions are often the 
only available mode for users and researchers. Black-box 
access implies that some properties of the model -- e.g., the presence of 
particular data in the (pre-)training corpus \citep{shokri2017membership, 
carlini2021extracting} -- are not known and must be 
inferred from the model's responses to queries. 
In our setting, we assume access to a collection 
of models for which the property is known and that we can relate the responses from a new model to the responses from the labeled models.
Our goal is to infer the property of a new / unlabeled model. 

Recent work has shown that collections of black-box models can be represented in Euclidean space by comparing embedded responses to a set of queries \citep{helm2024tracking, acharyya2024consistent}, and that these representations enable consistent model-level inference \citep{helm2025statistical}. 
However, existing results treat all query sets as equally informative. In practice, the signal carried by a query set varies widely: sets containing probes aligned with the inference task (``signal'') produce substantially more discriminative representations than those ``orthogonal'' to it \citep{helm2024tracking,helm2025statistical}.
This distinction has remained informal. 

We introduce a framework, the \textit{discriminative factorization}, to formalize the signal/orthogonal distinction.
The framework enables bounding the classification risk for a given query budget, estimating the number of queries required to achieve a given risk, and selecting high signal query sets for classification. 
We apply this framework to a synthetic problem and three system auditing tasks: detecting sensitive fine-tuning data, identifying covert persuasion bias via system prompts, and finding RAG systems with access to restricted documents. 
Our contributions are the following:
\begin{enumerate}[itemsep=0pt, topsep=0pt, parsep=2pt]
        \item \textbf{The discriminative factorization.} We 
    introduce the discriminative factorization and 
    decompose query-model interaction into discriminative directions.
    The factorization separates the cross-model information into query-dependent weights and model-pair-dependent sensitivities, enabling a formal notion of signal/orthogonal queries.
    In settings where the embedding function is injective, the discriminative factorization yields a bound, exponential in the query budget, on the probability that classification error exceeds the Bayes risk, governed by the prevalence of ``orthogonal'' queries. 
    In other settings, a similar rate applies to chance-level classification.
    Further, we provide a sufficient number of queries to perform better-than-chance (or, when the embedding function is injective, to classify with a particular risk).  
    \item \textbf{Estimation from spectral structure.} 
    We show that the properties of the discriminative factorization --
    the discriminative rank and the zero-set probabilities -- can be estimated 
    from the SVD of a query-model matrix.  
    In each setting, we demonstrate that the estimated parameters predict empirical decay rates when expected.
    \item \textbf{Query selection without task-specific knowledge.}
    For all real tasks, we construct query sets using the estimated discriminative factorization. 
    In each case, the constructed signal/orthogonal sets reproduce the oracle ordering, improving classification efficiency without task-specific knowledge of which queries are informative. 
\end{enumerate}

\subsection{Related work}\label{sec:related}
\paragraph{Low-dimensional representations of generative models.}
Low-dimensional representations of generative models can be constructed from internal activations \citep{raghu2017svcca, kornblith2019similarity, huh2024platonic, duderstadt2023comparing} and model weights \citep{aghajanyan2021intrinsic, chen2025extracting}.
These approaches require access to model internals, which are unavailable in the black-box setting.
Black-box approaches 
address this constraint by constructing representations from embedded responses to a query set via multidimensional scaling \citep{helm2024tracking, acharyya2024consistent, helm2025statistical, helm2026queryefficientmodelevaluationusing}.
These representations can then be used for model-level inference \citep{helm2025statistical, helm2026queryefficientmodelevaluationusing}.  
We build on this foundation by addressing how the query distribution affects the quality of the representations for a given task.

\paragraph{Inference on models.}
Prior black-box analyses focus primarily on individual models. 
For example, membership inference \citep{shokri2017membership, duan2024membership} 
and training data extraction \citep{carlini2021extracting} test whether 
specific examples influenced a single model, but typically require 
token-level log-probabilities. 
Pretraining data detection \citep{shi2024detecting, maini2024llm} requires either similar assumptions or explicit examples of data in the training set.
Datamodels \citep{ilyas2022datamodels} predict model outputs as a function of training data composition, but require retraining many models with known data subsets instead of solely using black-box queries. 
We provide theoretical guarantees connecting query budget to classification risk when using low-dimensional representations from embedded responses and demonstrate that use of our framework can directly improve practical query efficiency. 
\section{Motivation and setting}
\label{sec:setting}

Let ${\mathcal F}:=\{ f: {\mathcal Q} \to {\mathcal X} \}$ be the space of models. 
Each $f  \in {\cal F}$ is a random mapping from a finite input space $\mathcal{Q}$ with $|\mathcal{Q}| = M < \infty$ to a finite output space $\mathcal{X}$ with $|\mathcal{X}| = V < \infty$. 
We refer to $q \in \mathcal{Q}$ as a \textit{query} and the model’s output $ f (q)$ as a \textit{response}.
We assume responses $ f(q)_{1}, \hdots, f(q)_{K} $ are i.i.d.\ samples from $ P_{f}(q) $.

In the black-box setting, the model $ f $ is completely characterized by its response distributions, $ f = \{P_{f}(q)\}_{q\in\mathcal{Q}} $, and $ f \neq f' $ if and only if there exists $ q \in \mathcal{Q} $ such that $P_{f}(q) { \neq } P_{f'}(q)$.
In practice, analysis directly on $\mathcal{X}$ is often inconvenient or intractable.
Instead, for a given embedding function $ g: \mathcal{X} \to \mathbb{R}^{p} $, we work with embedded responses $ g(f(q))_{1}, \hdots, g(f(q))_{K} $ assumed to be i.i.d.\ from $ P_{f}^{(g)}(q) $.
Unless otherwise stated, we refer to $ P_{f}^{(g)}(q) $ as $ P_{f}(q) $ for the remainder of this paper.

To capture the dissimilarity between models, for a multiset of queries $ Q \subseteq_{m} \mathcal{Q} $ we define
\begin{equation*}
\label{eq:metric-on-models}
    d_{Q}(f, f') := \sqrt{ \sum_{q \in Q}  \frac{1}{m} d_{P}^2 (P_{f}(q), P_{f'}(q)) }
\end{equation*}
where $d_{P}$ is a metric on $\mathcal{P}_1(\mathbb{R}^p)$ such that $d_P^2$ is of negative type (such as the energy distance \citep{szekely2013energy, rizzo2016energy}).
As a nonnegative combination of such terms, $d_Q^2$ is also of negative type, so $d_Q$ embeds isometrically in a Hilbert space \citep{schoenberg1938metric}.

\paragraph{Query-dependent model representations.}
For a collection of black-box models $f_{1},\dots,f_{n}$ and multiset of queries $Q$, we define the pairwise distance matrix $D:= D_{i,i'} = d_{Q}(f_{i}, f_{i'})$. We then apply classical multidimensional scaling (MDS) \citep{torgerson1952multidimensional} to $D$ to obtain 
representations $\psi_Q(f_1), \ldots, \psi_Q(f_n) \in \mathbb{R}^d$. These are vector representations of the models $f_i$ with respect to $g$ and $Q$.
For some choices of $ d_{P} $, estimates of the representations $ \psi_{Q}(f_{i}) $ are consistent for their population counterpart \citep{acharyya2024consistent} and provide consistent model-level inference \citep{helm2025statistical}. 

Since $g(\mathcal{X}) \subset \mathbb{R}^{p}$ is finite, each $P_f(q)$ is a categorical distribution supported on at most $V$ atoms in $\mathbb{R}^p$. And, since $d_{Q}^2$ is negative type, MDS achieves zero stress for any $n<\infty$ for $d \geq \min\{ n-1, m(V-1)\}$ \citep{szekely2004testing, schoenberg1938metric}.

\subsection{Motivating example: detecting sensitive training data}\label{sec:sensitive}

\begin{figure}
    \centering
    \includegraphics[width=0.95\linewidth]{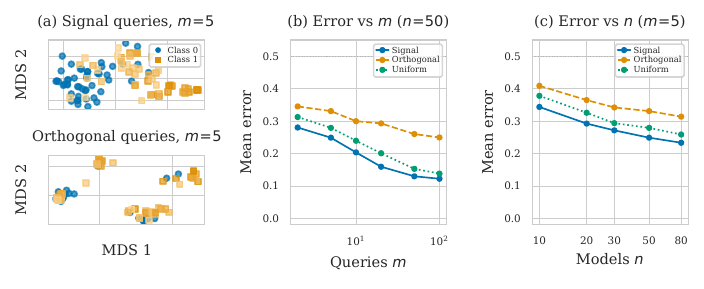}
    \caption{
    Classification of models by the presence of particular fine-tuning data. 
    An oracle knows which queries probe the axis that distinguishes the two classes (Signal) and which do not (Orthogonal).
    (a) Two-dimensional MDS embeddings using $m = 5$ signal queries (top) and orthogonal queries (bottom). 
    Each dot is a model.
    (b) Classification error as a function of query budget $m$ with $ n = 50 $ for signal, orthogonal, and uniform query sets. 
    (c) Classification error as a function of number of models with $ m = 5 $ for the same query sets. 
    The relative discriminative quality of the three query sets is stable across $ m $ and $ n $.
    Reported errors are the average of 500 different random samples.
    }
    \label{fig:motivating}
\end{figure}

To demonstrate the impact of query set on model representations and subsequent inference, we use a task inspired by \citet{helm2025statistical}:
using only black-box access and a collection of labeled models, can we identify whether a particular topic appeared in a model's fine-tuning corpus?

To investigate, we train 100 LoRA adapters from \texttt{Qwen2.5-1.5B-Instruct} on subsets of Yahoo Answers \citep{zhang2015character}: 50 on non-sensitive topics (class 0) and 50 on mixtures containing ``Politics \& Government'' data (class 1). 
We construct model representations with respect to three input query sets $\mathcal Q$:
\textit{Signal} queries directly from ``Politics \& Government'' data,
\textit{Orthogonal} queries on topics not included in any of the fine-tuning datasets, 
and \textit{Uniform}, the union of the two. 

For each of the three groups, we sample $m$ queries uniformly and construct $ \psi_{Q}(f_{1}), \hdots, \psi_{Q}(f_{n+1})$ with embedding dimension $d = 8$.
We train a random forest classifier with default parameters \citep{pedregosa2011skl} on the $n$ random training representations and evaluate on the remainder.
This process is repeated 500 times (see Appendix~\ref{app:motivating} for additional details).

Figure~\ref{fig:motivating} illustrates the phenomenon of interest.
Panel (a) shows embeddings in the first two dimensions for a single instance of $m=5$ queries from the Signal and Orthogonal sets. Signal queries yield strong separation of classes while separation from Orthogonal queries is much weaker.  
Panel (b) quantifies this observation via the classification error on the test models. 
The relative ordering matches intuition for all query budgets: Signal sets outperform Uniform sets, which outperform Orthogonal sets. 
Orthogonal query sets still enable better-than-chance classification on average. 
Panel (c) shows that this ordering is stable across the number of models for a fixed budget, with a decrease in error as the size of the training sample increases. 

In practice, the collection of queries that probe the discriminative axes of a classification task is not known \emph{a priori}. 
We introduce a framework where the query-model relationship is naturally represented as a matrix.
The spectral properties of this matrix, which can be estimated from the available labeled data, separate Signal and Orthogonal queries -- thereby improving the query efficiency of inference on models without prior task-specific knowledge.

\subsection{Model classification and risk}\label{sec:classifcation}
Let ${\cal Y} = \{ 0,\dots,C-1 \}$ be the set of model classes.  
We observe a collection of models and labels $(f_1, y_1), \ldots, (f_n, y_n) \in \mathcal{F} \times \mathcal{Y}$ drawn i.i.d.\ from $P_{fY}$.
We also observe an unlabeled model $f_{n+1}$.
Our goal is to predict the corresponding label $y_{n+1}$.

The risk of a classifier $h: \mathcal{F} \to \mathcal{Y}$ under 0-1 loss is given by
$L_{fY}(h) = \mathbb{P}_{fY}\left[h(f) \neq y\right]$. 
The \textit{Bayes risk} of a classification problem, denoted $L^*(P_{fY})$, is the infimum risk of classifiers $\{h: \mathcal{F} \to \mathcal{Y}\}$. 
For balanced binary classes, when $L^*(P_{fY}) < 0.5$, there exists a model-level classifier that performs better than chance.

Classification on the space of models $\mathcal{F}$ is impractical. 
Instead, we consider the proxy problem in the low-dimensional space induced by $ Q $.
That is, we consider $h_d \in \mathcal{H}_d = \{h_d: \mathbb{R}^d \to {\cal Y}\}$ trained on the representations $(\psi_Q(f_1), y_1), \ldots, (\psi_Q(f_n), y_n)$ that predicts the label of $ f_{n+1} $ via $\psi_{Q}(f_{n+1})$.
The risk of $h_d$ is given by
$L_{\psi Y}^{(n)}(h_d, Q) = \mathbb{P}_{(f_{n+1},y_{n+1})}\!\left[h_d(\psi_{Q}(f_{n+1})) \neq y_{n+1} \mid g, Q, (f_1, y_1), \ldots, (f_n, y_n)\right]$.

Our main results (\S\ref{sec:main-results}) show that under mild assumptions, the low-dimensional representations suffice for better-than-chance classification and, when $ g $ is injective, for Bayes-optimal classification.

\section{The discriminative factorization}
\label{sec:framework}
We now introduce the discriminative factorization.
Section~\ref{sec:sensitive} showed that query sets that probe the discriminative axes induce better representations for classification. 
To formalize this, 
we decompose the content of a query into discriminative directions.

\begin{definition}\label{ref:disc-factor}
Consider a model space ${\mathcal F}$, 
embedding function $g$,
and query set ${\mathcal Q}$. 
If $\alpha$ and $\phi$ are maps $\alpha: \mathcal{Q} \to [0, 1]^r$ and
$\phi: {\mathcal F} \times {\mathcal F} \to [0, \infty)^r$ 
such that for all $f, f' \in \mathcal{F}$ and $q \in \mathcal{Q}$
\begin{equation*}
    d_{P}^2(P_f(q), P_{f'}(q)) = \sum_{\ell=1}^{r} \alpha_\ell(q)\, \phi_\ell(f, f')
\end{equation*}
then we say ($\alpha$, $\phi$) admit a \textup{discriminative factorization} of rank $r$ for $({\cal F}, {\cal Q}, g, d_{P})$. This factorization is \textit{minimal} if no factorization of rank $r' < r$ exists.
\end{definition}

For each $q$, the \textit{weight} $\alpha_\ell(q)$ quantifies the intensity of discriminative signal along \textit{direction} $\ell$, while $\phi_\ell(f, f')$ captures the \textit{sensitivity} of the model pair $(f, f')$ to that direction.
A query is in the \textit{zero-set} of direction $\ell$ if $ \alpha_\ell(q) = 0$; let $ \mathcal{Z}_{\ell} := \{q: \alpha_\ell(q) = 0\} $. 
For a query distribution $\Pi_Q$ on $\mathcal{Q}$, the \textit{zero-set probabilities} are given by $\rho_{\ell} = \Pi_{Q}({\cal Z}_{\ell})$. 
A query $ q \in  \bigcap_\ell \mathcal{Z}_\ell$ is \textit{orthogonal} to the task. This formalizes the notion of ``orthogonal'' discussed in \S\ref{sec:sensitive}.

A discriminative factorization always exists with $r \leq M$: For any enumeration ${q}_{1}, \ldots, {q}_{M}$ of $\mathcal{Q}$, take $\alpha_\ell(q) = \mathbbm{1}{[q = {q}_{\ell}]}$ and $\phi_\ell(f, f') = d_{P}^2(P_f(q_\ell), P_{f'}(q_\ell))$.
With observed models $f_1,\ldots, f_n$, if $\sum_{q \in Q} \alpha_\ell(q) > 0$ for all $\ell$, 
the matrix of squared distances $[d_Q^2(f_i, f_{i'})]_{i,i'}$ is a conic combination of the $r$ matrices $\Phi_\ell = [\phi_\ell(f_i, f_{i'})]_{i,i'}$ and MDS achieves zero stress for some $d \leq \sum_\ell \mathrm{rank} (J \Phi_\ell J)$, where $J$ is the centering matrix. If each $\phi_\ell(f, f') = (x_\ell(f) - x_\ell(f'))^2$, each term has rank 1 and this reduces to $d \leq r \ll m(V-1)$.
More generally, if each $\phi_\ell(f, f') = \lVert x_\ell(f) - x_\ell(f') \rVert_2^2$ for $x_\ell(f) \in \mathbb{R}^{k_\ell}$, then $d \leq \sum_\ell k_\ell$.


\paragraph{Relation to other factorizations and decompositions.} Definition \ref{ref:disc-factor} is a query-indexed low-rank factorization of a collection of response dissimilarities. Like Fourier or Mercer expansions~\cite{mercer1909}, it writes a function as a weighted sum of terms, but its terms $\phi_\ell$ are not fixed in advance and both factors are nonnegative. Represented as a matrix over $\mathcal{Q} \times (\mathcal{F} \times \mathcal{F})$, it is a rank-$r$ nonnegative matrix factorization~\cite{lee1999nmf}. Section~\ref{sec:estimation} estimates its column space with the singular value decomposition, which gives the best rank-$r$ approximation in Frobenius norm~\cite{eckart1936} but with mixed-sign factors. The closest classical model is INDSCAL~\cite{carroll1970indscal}, a form of three-way multidimensional scaling that decomposes a stack of dissimilarity matrices, one per subject, into shared dimensions with nonnegative subject-specific weights. When each $\phi_\ell(f, f') = (x_\ell(f) - x_\ell(f'))^2$ for a scalar function $x_\ell$, Definition \ref{ref:disc-factor} is exactly INDSCAL with queries in the role of subjects. To our knowledge, this factorization has not previously been used to study collections of black-box functions with random outputs.

\subsection{Theoretical results}
\label{sec:main-results}
For our theoretical analysis, we let $\mathcal{Y} = \{0, 1\}$ and assume balanced classes $P_{fY}(y = 1) = 0.5$.
We further assume $L^*(P_{fY}) < 0.5$.
We let $\alpha$ and $\phi$ admit a minimal discriminative factorization of rank $r$ with zero-set probabilities $\rho_{1},\dots,\rho_{r}$.
Our results depend on rank $r$ and the $\{\rho_{\ell}\}$, which are characteristic of the task and query distribution $\Pi_{Q}$.  
We require that each $ P_{f}(q) $ have finite first moment and make mild assumptions on $\Pi_Q$ and the quality of $ g $:
\begin{assumption}[Direction-covering query distribution]\label{assum:direction-covering}
For each direction $\ell \in \{1, \ldots, r\}$, the query distribution $\Pi_Q$ places positive mass outside the zero set: $\rho_\ell = \Pi_Q(\mathcal{Z}_\ell) < 1$.
\end{assumption}

\begin{assumption}[Non-degenerate embedding]\label{assum:positive-rank}
The embedded class-conditional distributions differ: $P_{g(f) \mid y = 0} \neq P_{g(f) \mid y = 1}$, where $g(f) := \{P_f^{(g)}(q)\}_{q \in \mathcal{Q}}$.
\end{assumption}

Given the high dimensionality of semantic spaces, it is extremely unlikely that a given query is in the zero set of a particular direction.\footnote{ 
e.g., a generic Sports query may be related to a task determining the presence of Politics \& Government fine-tuning data through semantic leak from legislation on sports betting.
}
Hence, Assumption \ref{assum:direction-covering} will be satisfied by most distributions on query sets. 
Assumption~\ref{assum:positive-rank} ensures that the embedding function $ g $ does not destroy all discriminative information available at the model level. 
For balanced classes it is equivalent to assuming $L^*(P_{g(f)Y}) < 0.5$, since $L^*(P_{g(f)Y}) = \tfrac{1}{2}\bigl(1 - \mathrm{TV}(P_{g(f) \mid y=0}, P_{g(f) \mid y=1})\bigr)$.

Our first result shows that the MDS representation induced by a large-enough query set suffices for better-than-chance classification under Assumption \ref{assum:direction-covering}
and Assumption~\ref{assum:positive-rank}. 
Full proofs of all results in this section are provided in Appendix \ref{app:theory}.

\begin{theorem}\label{thm:main}
Let $P_{fY}$ satisfy Assumption~\ref{assum:positive-rank} and $Q = \{q_1, \dots, q_m\}$ a multiset with i.i.d. queries from $\Pi_Q$ satisfying Assumption~\ref{assum:direction-covering}. 
Then there exists $n^* \in \mathbb{N}$ such that for all $n \geq n^*$ 
\begin{equation*}
    \mathbb{P}_{Q, (f_{i}, y_{i})}\!\left[\inf_{h_d \in \mathcal{H}_d} L_{\psi Y}^{(n)}(h_d, Q) \geq 0.5\right]
    \leq \sum_{\ell=1}^{r} \rho_{\ell}^m + \gamma_\epsilon(n),
\end{equation*}
for $d \geq \min \{ r^2, n, m(V-1) \}$, where $\gamma_\epsilon(n) \to 0$ as $n \to \infty$.
\end{theorem}

The bound in Theorem~\ref{thm:main} makes extremely mild assumptions on the embedding function $g$. The quality of $g$ is absorbed into $\{\alpha_\ell \}$ and
$\{ \rho_{\ell}\}$.
When the embedding function $g$ is injective\footnote{While modern sentence embedding models \citep{reimers2019sentence,zhai2023siglip,nussbaum2024nomic, radford2021clip, oquab2024dinov2} are not injective, they are practically injective in that the number of inputs they map to the same element of $ \mathbb{R}^{p} $ is small.}, a stronger result holds: the MDS representation asymptotically achieves the Bayes risk.
\begin{theorem}\label{thm:bayes}
In the setting of Theorem~\ref{thm:main}, suppose $g$ is injective.
For any $\epsilon > 0$ there exists $n^*_\epsilon \in \mathbb{N}$ such that for all $n \geq n^*_\epsilon$,
\begin{equation*}
    \mathbb{P}_{Q,F,(f,y)}\!\left[\inf_{h_d \in \mathcal{H}_d} L_{\psi Y}^{(n)}(h_d, Q) \geq L^*(P_{fY}) + \epsilon\right]
    \leq \sum_{\ell=1}^{r} \rho_{\ell}^m + \gamma_\epsilon(n),
\end{equation*}
for $d \geq \min\ \{ r^2, n, m(V-1) \}$, where $\gamma_\epsilon(n) \to 0$ as $n \to \infty$.
\end{theorem}
The bounds 
in both Theorem~\ref{thm:main} and Theorem~\ref{thm:bayes} 
hold for any finite $m$. 
Letting $\rho = \max_\ell \rho_{\ell}$, the first term becomes $\sum_{\ell=1}^{r} \rho_{\ell}^m \leq r \rho^m$.
When $r= M$, 
we get $M(1-\min_{q}\Pi_{Q}(q))^{m}$, 
recovering the coupon-collector scaling in which the entire query space must be covered \citep{erdHos1961classical}. 

The term $\gamma(n)$ captures the sample complexity of learning in the representation space $\mathbb{R}^d$, with the specific rate highly dependent on the class-conditional distributions.
Indeed, 
by the ``no free lunch'' theorem \citep{devroye1996probabilistic}, for any classifier such that $\gamma(n) \to 0$ and fixed $ n $, there exist class-conditional distributions for which the expected performance is near chance. 

Explicit rates follow if the class-conditional distributions are restricted. For example, suppose the representations $\psi_Q(f)$ given $y = 0$ and $y = 1$ are distributed as $\mathcal{N}(\mu_0, \Sigma)$ and $\mathcal{N}(\mu_1, \Sigma)$. The Bayes rule is linear with $L^* = \mathbb{P}(Z > \Delta/2)$ for a standard normal $Z$, where $\Delta = \sqrt{(\mu_1 - \mu_0)^\top \Sigma^{-1} (\mu_1 - \mu_0)}$ is the Mahalanobis separation between the classes. The plug-in linear discriminant has expected excess risk $O(d/n)$, with a constant that depends on $\Delta$~\cite{okamoto1963, mclachlan1992}. By Markov's inequality, $\gamma_\epsilon(n) = O(d / (n\epsilon))$ in Theorem~\ref{thm:bayes}, and the same rate holds for $\gamma(n)$ in Theorem~\ref{thm:main} with $\epsilon = 1/2 - L^*$.

We now state a corollary that makes the query sample complexity explicit. 
Using this result, it is possible to calculate the query budget required for given error tolerance. 

\begin{corollary}\label{cor:high-prob}
    Under the conditions of Theorem~\ref{thm:main} or Theorem~\ref{thm:bayes}, let $\rho=\max_{\ell}\rho_{\ell}$. For all $\delta > 0$, there exist $m^* \in \mathbb{N}$ and $n^* \in \mathbb{N}$ such that for all $m \geq m^*$ and $n \geq n^*$
    \begin{equation*}
                \mathbb{P}_{Q, (f_{i},y_{i})} \left[\inf_{h_{d} \in {\mathcal H}_{d} } L_{\psi Y}^{(n)}(h_d, Q) \geq {\mathcal L} \right] \leq \delta,
    \end{equation*}
    where ${\mathcal L} = 0.5$ under Theorem~\ref{thm:main} and ${\mathcal L} = L^*(P_{fY}) + \epsilon$ for fixed $\epsilon > 0$ under Theorem~\ref{thm:bayes}. In particular, $m^* = \left\lceil \frac{\ln(\delta/2r)}{\ln \rho} \right\rceil$ and any $n^*$ such that $\gamma_{(\epsilon)}(n) \leq \delta/2$ for all $n \geq n^*$ suffice.
\end{corollary}

\subsection{Estimating the discriminative rank and zero-sets}\label{sec:estimation}
Our results depend on $r$ and $\{\rho_\ell\}$. 
In practice, they must be estimated from the $n$ labeled models and their responses to a pilot set of $m$ queries $Q$. 
To do so, we construct the $m \times n^2$ sample \textit{Query-Model Matrix} 
${E}:= {E}_{j,(i,i')} = d_P^2(P_{f_i}(q_j), P_{f_{i'}}(q_j))$ 
for each $(f_i, f_{i'}) \in \{f_1,...,f_n\}^2$ and $q_j \in Q$.

At temperature zero, each $P_f(q)$ is a point mass and $E$ is computed exactly;
at positive temperature each entry is estimated from repeated responses, e.g., with the U-statistic for the energy distance \citep{szekely2013energy}.

In the black-box setting, each query distribution $P_f(q)$ 
corresponds to a feature of the model. 
Finding signal queries thus corresponds to finding informative features. 
Since $E = \boldsymbol{\alpha}\,\boldsymbol{\Phi}^\top$ has rank $r$ under the minimal factorization, the number of nonzero singular values of $E = U\Sigma W^\top$ identifies $r$, 
and the first $r$ left singular vectors $U_r$ span the column space of $\boldsymbol{\alpha}$: $\boldsymbol{\alpha} = U_r R$ for an invertible $R$.
Hence $q_j$ is orthogonal if and only if the $j$th row of $U_r$ is zero, and for $r=1$, $|U_{j,1}| \propto \alpha_1(q_j)$.
For $r>1$ the individual directions are identified only up to $R$; when each query loads on few directions, a varimax rotation of $U_r$ can recover them \citep{rohe2023varimax}.

We estimate the discriminative rank via the spectral gap of $E$: $\hat{r} = \mathrm{argmax}_s \; \sigma_s / \sigma_{s+1}$.
Given $\hat{r}$, we let $\hat{\boldsymbol{\alpha}}$ be the first $\hat{r}$ left singular vectors of $E$, rotated by varimax when $\hat{r} > 1$.
Given $\hat{\boldsymbol{\alpha}}$, we let
\begin{equation*}
    \hat{\mathcal{Z}}^{(\tau)}_\ell = \left\{ q_j \in Q: 
    |\hat{{\alpha}}_{\ell}(q_j)|    
    < \tau 
    \right\}
\end{equation*}
for $\tau > 0$. 
In the experiments below, we set the threshold by fitting a two-component Gaussian mixture model on $\{|\hat{\alpha}_\ell(q_j)|\}_{j=1}^m$ and let $\hat{\rho}_{\ell}$ be the mixture weight of the component closer to zero;
BIC \citep{raftery1995bayesian} against a one-component model supports a near-zero component in all three real tasks (Figure~\ref{fig:estimation}). 
The estimated Orthogonal set is $\bigcap_\ell \hat{\mathcal{Z}}_\ell$. 
Writing the observed matrix as $E + N$, where $N$ contains sampling noise and departures from an exact rank-$r$ factorization, 
Wedin's theorem \citep{wedin1972perturbation} gives $\lVert \sin\Theta(\hat{U}_r, U_r) \rVert_2 \leq \lVert N \rVert_2 / (\sigma_r - \lVert N \rVert_2)$, so the estimates are stable while $\lVert N \rVert_2$ is small relative to $\sigma_r$.

\section{Experiments}\label{sec:experiments}
All experiments share the same pipeline: temperature-zero generation, sentence 
embedding, energy distance, classical MDS, and random forest classification
\footnote{The theorems bound the performance of the \textit{best} classifier in $\mathcal{H}_d$. The random forest may not achieve this infinum.
The empirical results are thus conservative estimates of the theoretical quantities.}.
We note that our experimental results focus on LLMs but that the methods apply analogously to any collection of black-box models with a suitable dissimilarity defined on their outputs. 

\subsection{Synthetic validation}\label{sec:synthetic}
We first consider classification problems where, by construction, the discriminative factorization holds exactly and all theoretical quantities are known.
For each of $M = 100$ queries, we define $\alpha_\ell(q) = \xi_{q\ell} \cdot w_{q\ell}$ where 
    $\xi_{q\ell} \sim \mathrm{Bernoulli}(1 - \rho)$ and
    $w_{q\ell} \sim \mathrm{Uniform}(0,1)$, 
    so that $\rho$ is the zero-set probability. 
    So that all $r$ directions are necessary for classification,
each model $f$ has latent type $\theta_f \in \{0,1\}^r$ drawn uniformly
with class label $y = \mathrm{parity}(\theta_f)$. 
Per-query distances are constructed directly as 
$d_P^2(P_{f_i}(q), P_{f_j}(q)) = \sum_\ell \alpha_\ell(q) \cdot 
\mathbbm{1}[\theta_{i,\ell} \neq \theta_{j,\ell}]$, giving 
$\phi_\ell(f_i, f_j) = \mathbbm{1}[\theta_{i,\ell} \neq \theta_{j,\ell}]$.
Unless otherwise noted, $n = 100$, $r = 5$, and $\rho = 0.7$. 
All results are averaged over 1{,}000 repetitions (5{,}000 for panel~(d)).

\begin{figure}
    \centering
    \includegraphics[width=0.95\linewidth]{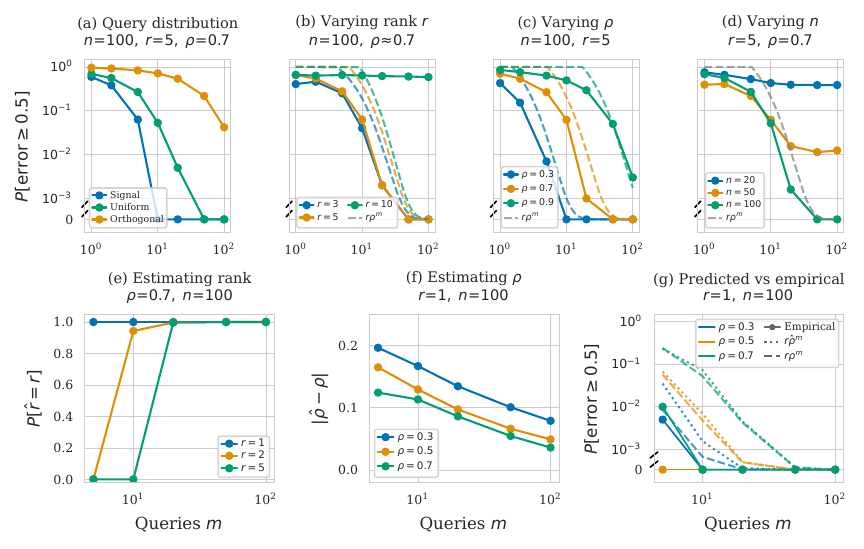}
    \caption{Synthetic validation. 
\textbf{Top row 
(Theorem~\ref{thm:main})
:} Failure probability 
$P[\mathrm{err} \geq 0.5]$ as a function of query budget $m$, varying 
(a) query distribution, (b) discriminative rank $r$, (c) zero-set 
probability $\rho$, and (d) number of labeled models $n$. 
Dashed lines show the theoretical bound $r\rho^m$. 
\textbf{Bottom row (estimation):} Recovery of discriminative 
factorization parameters from $E$ constructed with the same $m$ 
queries: (e) probability of correct rank recovery, (f) absolute error 
$|\hat{\rho} - \rho|$ of the zero-set probability estimate, and (g) empirical failure probability (solid) compared to the predicted bound using estimated parameters $\hat{r}\hat{\rho}^m$ (dotted) and true parameters $r\rho^m$ (dashed). The estimated bound is conservative but tracks the correct decay rate. 
All results averaged over 1{,}000 repetitions (5{,}000 for panel~(d)).}
    \label{fig:synthetic}
\end{figure}

The top row of Figure~\ref{fig:synthetic} varies the parameters of the bound $r\rho^m + \gamma(n)$ from Theorem~\ref{thm:main} one at a time. 
Panel~(a) compares Signal, Uniform, and Orthogonal query distributions:
the ordering matches the example in \S\ref{sec:sensitive}, with signal queries dominating and orthogonal queries remaining near chance. 
Panels~(b) and~(c) vary $r$ and $\rho$ respectively. 
The failure probability decays exponentially in $m$ with 
slope $\log \rho$ and intercept scaling with $r$, matching the bound. 
For $r = 10$ the curve stays near $0.5$ because $n = 100 \ll 2^{10}$, the number of possible feature vectors,
illustrating a case where the sample part of the bound 
$\gamma(n)$ dominates and no query budget suffices. 
Panel~(d) varies $n$ at fixed $\rho = 0.7$. 
For each $n$, failure probability decreases in $m$ but saturates at a floor determined by $n$, showing that additional queries cannot compensate for too few labeled models.

The bottom row validates the estimation procedure from 
\S\ref{sec:estimation}. 
Panel~(e) demonstrates that the spectral gap estimator is reliable for moderate $r$ with moderate $m$: 
$P[\hat{r} = r]$ reaches 1.0 by 
$m \approx 20$ for all $r$. 
Panel~(f) shows zero-set probability estimation.
$|\hat{\rho} - \rho|$ decreases with $m$ for all values of $\rho$.
There is a mild conservative bias, with better estimation for larger $\rho$. This is likely due to the GMM method assigning weak-signal queries to the zero set. 
Panel~(g) compares the empirical failure probability (solid) to the predicted bound. For each $\rho$, the bound using estimated parameters $\hat{r}\hat{\rho}^m$ (dotted) closely tracks the true bound $r\rho^m$ (dashed). 
The gap 
shrinks as both $\rho$ and $m$ grow.
Importantly, the predicted bounds control the empirical probability for all $m$.

\begin{figure}
    \centering
    \includegraphics[width=0.95\linewidth]{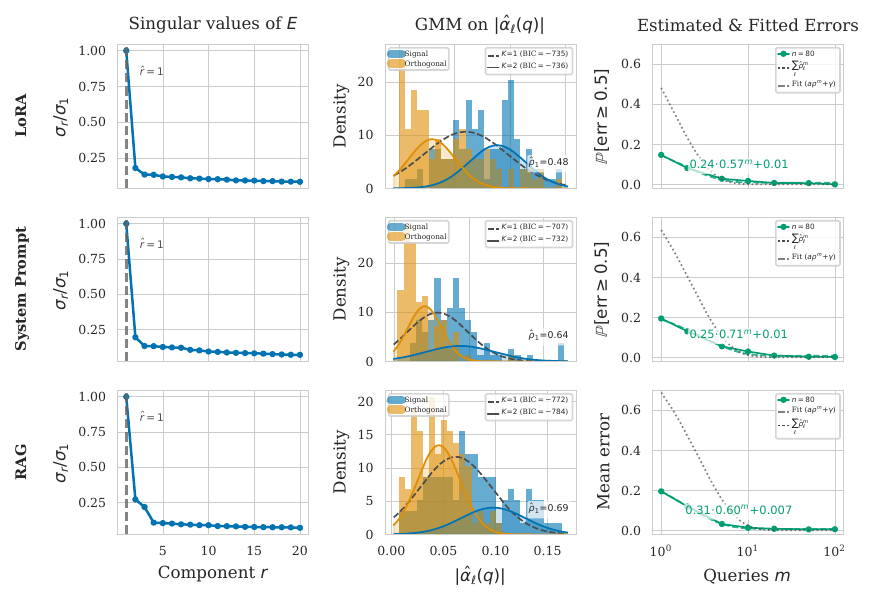}
    \caption{Estimation and validation of the discriminative factorization on three model classification tasks. 
    \textbf{Left:} Singular value ratios $\sigma_r / \sigma_1$ of the query-model matrix $E$, with the spectral gap identifying $\hat{r} = 1$ in all three settings. 
    \textbf{Center:} GMMs fitted to the estimated per-query discriminative loadings $|\hat{\alpha}_{\ell}(q)|$.  
    Queries are colored by their estimated GMM component.
    The near-zero component aligns with orthogonal queries in all three tasks. 
    \textbf{Right:} Empirical failure curves with uniform sampling over all queries for the task (solid), the estimated bound $\hat{\rho}^m$ from the GMM (dotted), and fitted curves $a\rho^m + \gamma$ (dashed) for $n = 80$. 
    The first two rows show $P[\mathrm{err} \geq 0.5]$ (Theorem~\ref{thm:main}). 
    The bottom row shows mean error (Theorem~\ref{thm:bayes}).}
    \label{fig:estimation}
\end{figure}

\subsection{Real-data tasks}\label{sec:real-data}
We evaluate the framework on three system auditing tasks.
Each covers one of three common ways to change the behavior of a system (parameter-efficient fine-tuning \citep{han2024parameterefficientfinetuninglargemodels}, system prompts \citep{sahoo2025systematicsurveypromptengineering}, and retrieved context \citep{gao2024retrievalaugmentedgenerationlargelanguage}). 
The first is the LoRA task from \S\ref{sec:sensitive}. 
The other two parameterize systems via System Prompts and retrieval datasets.

\paragraph{System prompt auditing.}
We consider 100 generation configurations of \texttt{mistral-small}, each with a distinctive system prompt.
The difference between Class 0 and Class 1 is a single sentence instructing biased models to prefer natural, time-tested approaches when giving advice (50 neutral, 50 biased). 
Queries consist of 100 neutrally-framed recommendation questions (Signal) and 100 factual questions (Orthogonal). 
Full prompt templates are provided in Appendix~\ref{app:system-prompt}.

\paragraph{RAG compliance auditing.}
We consider
120 retrieval-augmented generation (``RAG'') configurations for
\texttt{ministral-8b}.
All models have access to a shared public knowledge base about a fictional company. 
Class~1 (compliant) models access only the public data. Class~0 models additionally access restricted document stores covering (fabricated) internal finance data, HR data, or both.  
Each model's document collection is a random subset of its accessible stores. 
Queries consist of 50 finance questions and 50 HR questions (Signal) and 100 general company questions (Orthogonal). 
Full details are provided in Appendix~\ref{app:retrieval}.

\subsubsection{Estimation validation}
We first demonstrate the application 
of the estimation procedure from \S\ref{sec:estimation} to each task (Figure~\ref{fig:estimation}). We then assess the estimated bound against the empirical error curves when sampling over all queries. 
The spectral gap of $E$  (left column) identifies $\hat{r} = 1$ in all cases. 
For the LoRA and System Prompt tasks, this matches the single designed discriminative axis. 
For the RAG task, the design includes two independent document domains (finance and HR).
The recovery of $\hat{r} = 1$ rather than $\hat{r} = 2$ reflects a dominant direction on which both finance and HR signal queries load ($\sigma_1/\sigma_2 = 3.7$); the finance-versus-HR contrast is weaker and appears in the second direction. 
In all three tasks, the GMMs fitted to the estimated weights $|\hat{\alpha}_{\ell}(q)|$ (center column) align queries in 
the near-zero component with Orthogonal queries and the positive component with Signal queries 
(adjusted Rand indices of $ 0.172, 0.182,$ and $0.149$; permutation test $p$-value $ < 0.001$ for all three tasks).
The estimation procedure therefore recovers the designed Signal/Orthogonal structure significantly above chance.
The panels showing the empirical error for $n=80$ (right column)  
reveal that
the fitted decay rates are close to the estimated $\hat{\rho}$ values, while the fitted prefactors are consistently smaller than $\hat{r}$. 

\subsubsection{Classification with estimated query sets}
We next investigate whether the estimated Signal/Orthogonal query sets exhibit the same behavior as the oracle sets by comparing classification errors using the various query pools (Figure~\ref{fig:query-selection}).
For each of the three tasks,
we construct two new query sets from the estimated loadings $\hat{\alpha}_{\ell}(q)$: \emph{Estimated Signal} (queries in the GMM component farther from 0) and \emph{Estimated Orthogonal} (queries in the GMM component closest to 0). 
For each of the oracle Signal and Orthogonal and the {Estimated Signal} and {Estimated Orthogonal} query sets, we sample $m$ queries uniformly, train a classifier on $n=80$ models, and classify the remaining unlabeled models.
We report the mean classification error from 500 random training samples. 
\begin{figure}
    \centering
    \includegraphics[width=0.95\linewidth]{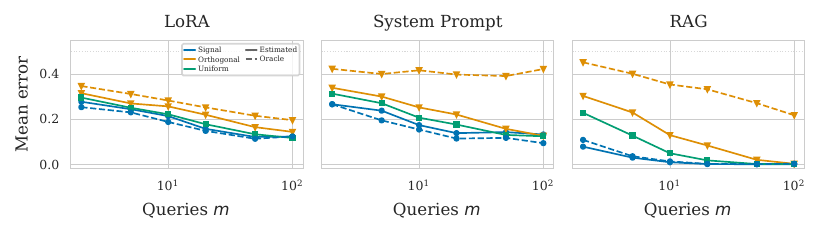}
    \caption{Classification error using estimated signal and orthogonal query sets on three model classification tasks. 
    Query sets are constructed from the estimated discriminative loadings $|\hat{\alpha}_{\ell}(q)|$ without knowledge of discriminative content.  
    The ordering est.\ signal $<$ uniform $<$ est.\ orthogonal recovers the phenomenon observed with known query types in \S\ref{sec:sensitive}.
    The oracle signal (resp.\ orthogonal) query set is better (resp.\ worse) than its estimated counterpart.}
    \label{fig:query-selection}
\end{figure}

For each task, the ordering Estimated Signal $\leq$ Uniform $\leq$ 
Estimated Orthogonal holds for all values of $m$, reproducing the 
phenomenon from \S\ref{sec:sensitive} without knowledge of 
query content. The effect is largest on the RAG task, where 
Estimated Signal at $n = 80$ reaches near-zero error by $m = 10$ 
while Estimated Orthogonal remains above $0.1$ at $m \approx 20$. 
Oracle Signal sets outperform their estimated counterparts in all 
three tasks. 
Notably, even oracle Orthogonal sets enable better-than-chance classification, suggesting that truly orthogonal queries are rare in practice -- consistent with the observation that most queries have nonzero $\hat{\alpha}_\ell(q)$ in Figure~\ref{fig:estimation}. If the embedded class-conditional distributions of the models were identical, no classifier could beat chance on held-out models. The above-chance accuracy on all three tasks, including with oracle Orthogonal queries, is therefore indirect evidence that the embedding function preserves the label information that Assumption~\ref{assum:positive-rank} requires.

\paragraph{Robustness.} The results above use one response dissimilarity, one embedding function, and one classifier. Appendix~\ref{app:robustness} repeats the estimation procedure with four other dissimilarities and three other embedding functions, and the classification experiment with three other classifiers. The estimated rank is $\hat{r} = 1$ in every setting, and the estimated Signal/Orthogonal partition agrees with the designed partition significantly above chance in every setting except two embedding functions on the System Prompt task. Accuracy increases with $m$ for all four classifiers on all three tasks.
\section{Discussion}

We introduced the discriminative factorization, a decomposition of 
the query-model interaction that formalizes the signal/orthogonal 
distinction observed in prior work and yields exponential-rate bounds 
on classification error. We showed that the discriminative rank and 
zero-set probabilities can be recovered from the SVD of a query-model 
matrix, that the estimated parameters predict empirical decay rates 
on three auditing tasks, and that query sets constructed from the 
estimated discriminative field reproduce the oracle signal/orthogonal 
ordering without task-specific knowledge.

We expect $r$ to be small in practice. 
In particular, if modern models are converging toward 
shared representations \citep{huh2024platonic, kaushik2025universalweightsubspacehypothesis}, then variation introduced by a 
specific modification -- fine-tuning on particular data, inserting a system 
prompt, or connecting a document store -- perturbs the base representation along 
a few task-relevant directions rather than globally. The discriminative rank $r$ captures exactly the dimensionality of this perturbation subspace. 
Our empirical results are consistent with this view as all three real tasks yield $\hat{r} = 1$. 

The discriminative factorization applies to any setting where a collection of black-box systems can be queried and their responses mapped to a numerical space. 
The framework only requires that the systems respond to shared queries.  
Our theoretical results suggest that any property not knowable from a single model in isolation can be inferred with better-than-chance accuracy (and, when the numerical space is injective with respect to the outputs, up to statistical noise) in the 
presence of a labeled collection.
More poignantly: a ``black-box" model, in the context of a labeled 
collection, is not fully black -- it is gray.

\subsection*{Limitations and future work}
\textit{(i) Zero-set idealization.} As seen in the distributions of 
the query loadings in the center column of Figure~\ref{fig:estimation}, the notion of a 
zero-set is an idealization: most queries have non-zero 
$\alpha_{\ell}(q)$. 
As such, the fitted prefactors are consistently 
smaller than $\hat{r}$, and estimated orthogonal queries still enable 
better-than-chance classification.
A bound incorporating signal magnitude would replace the zero-set 
probability $\rho_\ell$ with the probability that the accumulated 
signal $A_\ell(Q) = \sum_k \alpha_\ell(q_k)$ falls below a 
task-dependent threshold, yielding tighter prefactors at the cost 
of an additional parameter. 
We leave theoretical formulation and experimental validation of a bound of this type to future work.

\textit{(ii) Direction independence.} The bound 
$\sum_\ell \rho_\ell^m$ treats all directions as independent and 
equally necessary. In practice, directions may be correlated or 
unequally important for classification, as evidenced by the RAG 
experiment where two domains collapse to $\hat{r} = 1$.
Characterizing when tasks with $ r' $ semantic directions have $ r < r' $ discriminative directions is an important future direction.

\textit{(iii) Query selection.} The estimated discriminative field 
$\hat{\alpha}$ identifies which queries carry signal, but treats all queries of the signal class as equally informative. Formalizing a query selection method conditioned on the signal queries could yield further improvements to query efficiency, particularly for tasks with $r > 1$ where direction coverage matters.

\textit{(iv) Beyond binary classification.} The theoretical bounds and experiments address binary classification. 
Notions of orthogonality exist for other model-level tasks, as well as the low-dimensional representations themselves.
Our proposed estimation procedures apply to any task, though extending the bounds to continuous covariates or to the representations is non-trivial.
Extension to multi-class classification with non-uniform priors is possible with the current framework.

\subsection*{Acknowledgements}
We would like to thank Aranyak Acharyya, Avanti Athreya, and Brandon Duderstadt for helpful comments and discussions throughout the development of this manuscript.
We gratefully acknowledge funding from Defense Advanced Research Projects Agency (DARPA) Artificial Intelligence Quantified (AIQ) award number HR00112520026.

\newpage

\bibliographystyle{plainnat}
\bibliography{biblio}

\appendix
\section{Theoretical Results}
\label{app:theory}
In order to prove Theorem~\ref{thm:main}, we first prove Theorem~\ref{thm:bayes}.
\subsection{Bayes-optimal classification (Theorem \ref{thm:bayes})}
\begin{proof}
We are interested in the risk of a classifier $h_{d}^{(n)}$ trained on $n$ sample representations
\begin{equation*}
    \inf_{h^{(n)}_{d} \in {\mathcal H}_{d}}\mathbb{P}_{(f_{n+1},y_{n+1})}\!\left[h^{(n)}_{d}(\psi_{Q}(f_{n+1})) \neq y_{n+1} \mid g, Q, (f_1, y_1), \ldots, (f_n, y_n)\right] = \inf_{h_{d} \in {\mathcal H}_d} L^{(n)}_{\psi Y}(h_{d}, Q)    
\end{equation*}
Define 
\begin{equation*}
d_{{\mathcal Q} }(f,f') := \lim_{ m \to \infty } d_{Q}(f,f') =\sqrt{ \sum_{q \in {\mathcal Q} } \Pi_{Q}(q) d_{P}^2(P_{f}(q), P_{f'}(q)) } 
\end{equation*}
Denote $\mathcal{F}_g := \{ g(f) : f\in \mathcal{F} \}$ the set of embedded models and $P_{g(f)Y}$ the induced joint distribution. 
Let $A$ be the event that $\sum_{q \in Q} \alpha_{\ell}(q) > 0$ for all $\ell$. The complement event is that there exists an $\ell$ such that $\sum_{q \in Q} \alpha_{\ell}(q) = 0$, i.e., $q \in {\mathcal Z}_{\ell}$ for all $q \in Q$. For a single $\ell$ over $m$ queries, $\mathbb{P}[\alpha_{\ell}(q) = 0 \, \forall\,q \in Q] = \rho_{\ell}^m$. A union bound over all $r$ directions gives $\mathbb{P}[A^c] \leq \sum_{\ell=1}^r \rho_{\ell}^m$. 

Since $d_{Q}^2$ is of negative type, MDS achieves zero stress for finite $d\geq\min\{n-1, m(V-1)\}$ for all $n$ and exactly preserves the geometry of ${\mathcal F}_{g}$ with respect to $d_{Q}$. 
That is, the mapping $f_{i} \mapsto \psi_{Q}(f_{i}) \in \mathbb{R}^{d}$ is an isometry for all $n$ and all $\{ f_{i} \}_{i=1}^{n+1}$.

When $A$ holds, $f\neq f' \iff d_{{\mathcal Q}}(f,f') > 0 \iff d_{Q}(f,f') > 0$, i.e. $d_{Q}$ is a metric on ${\mathcal F}_{g}$. 
In this case, any distance-based universally consistent classifier (e.g. $k$ nearest neighbors with $k\to \infty$, $k/n \to 0$ as $n\to \infty$ \citep{devroye1996probabilistic}) is equivalent for $\{ f_{i} \}_{i=1}^n$ and $\{ \psi_{Q}(f_{i}) \}_{i=1}^n$. That is, there exist classifiers $\tilde{h}_{d}^{(n)}$ such that for any fixed $\epsilon > 0$:
\begin{align*}
    &\lim_{ n \to \infty } \mathbb{P}\left[\inf_{h_{d} \in {\mathcal H} _{d} }\mathbb{P}_{f_{n+1}, y_{n+1}}\left[{h}_{d}(f_{n+1}) \neq y_{n+1} \mid g, Q, (f_1, y_1), \ldots, (f_n, y_n)\right] - L^*(P_{g(f)Y}) \geq \epsilon \right] \\
    &\leq\lim_{ n \to \infty } \mathbb{P}\left[\mathbb{P}_{f_{n+1}, y_{n+1}}\left[\tilde{h}_{d}^{(n)}(f_{n+1}) \neq y_{n+1} \mid g, Q, (f_1, y_1), \ldots, (f_n, y_n)\right] - L^*(P_{g(f)Y}) \geq \epsilon \right] \\
    &=\lim_{ n \to \infty } \mathbb{P}\left[\mathbb{P}_{f_{n+1}, y_{n+1}}\left[\tilde{h}^{(n)}_{d}(\psi_{Q}(f_{n+1})) \neq y_{n+1} \mid g, Q, (f_1, y_1), \ldots, (f_n, y_n)\right] - L^*(P_{g(f)Y}) \geq \epsilon \right] \\
    &= \lim_{ n \to \infty } \mathbb{P}\left[L_{\psi Y}^{(n)}(\tilde{h}_{d}, Q) - L^*(P_{g(f)Y}) \geq \epsilon\right] \\
    &=0
\end{align*}
With $\gamma_{\epsilon}(n) := \mathbb{P}\left[L_{\psi Y}^{(n)}(\tilde{h}_{d}, Q) - L^*(P_{g(f)Y}) \geq \epsilon\right]$, we have $\gamma_{\epsilon}(n) \to 0$ as $n\to \infty$. Thus, we have 
\begin{align*}
    \mathbb{P}_{Q, (f_{i}, y_{i}) }\left[\inf_{h_{d}^{(n)} \in {\mathcal H}_{d}} L^{(n)}_{\psi Y}(h_{d}^{(n)}, Q) \geq L^*(P_{g(f)Y}) + \epsilon\right] &\leq \gamma_{\epsilon}(n) + \mathbb{P}[A^c] \\
    &\leq \gamma_{\epsilon}(n) + \sum_{\ell=1}^r \rho_{\ell}^m
\end{align*}
When $g$ is injective, $L^*(P_{g(f)Y}) = L^*(P_{fY})$ and the result follows.
\end{proof}

\subsection{Better-than-chance classification (Theorem \ref{thm:main})}
\begin{proof}
Following the proof of Theorem~\ref{thm:bayes}, we obtain 
\begin{equation*}
   \mathbb{P}_{Q, (f_{i}, y_{i})}\!\left[\inf_{h_d \in \mathcal{H}_d} L_{\psi Y}^{(n)}(h_d, Q) \geq L^*(P_{fY}) + \epsilon\right] \leq \gamma_{\epsilon}(n) + \sum_{\ell=1}^r \rho_{\ell}^m
\end{equation*}
with $\gamma_{\epsilon}(n) \to 0$ as $n \to \infty$. For balanced classes, $L^*(P_{g(f)Y}) = \tfrac{1}{2}\bigl(1 - \mathrm{TV}(P_{g(f) \mid y=0}, P_{g(f) \mid y=1})\bigr)$ \citep{devroye1996probabilistic}, so Assumption~\ref{assum:positive-rank} gives $L^*(P_{g(f)Y}) < 0.5$. Hence there exists $\eta > 0$ such that $L^*(P_{g(f)Y}) + \eta = 0.5$. Setting $\gamma(n) := \gamma_\eta(n)$ yields the desired bound.
\end{proof}

\subsection{Query complexity}
\begin{proof}
    This result is immediate from either Theorem~\ref{thm:main} or Theorem~\ref{thm:bayes}. Solving $r\rho^m = \frac{\delta}{2}$ and $\gamma_{(\epsilon)}(n) = \frac{\delta}{2}$ yields the particular values of $m^*$ and $n^*$.  
\end{proof}

\section{Additional Experimental Details \& Results}
\label{app:data-details}

\subsection{Motivating Experiment (LoRA)}
\label{app:motivating}

All training data is drawn from the Yahoo Answers Topics dataset 
\citep{zhang2015character}, which contains approximately 1.4M questions 
with best answers across 10 topic categories. We designate 
\textit{Politics \& Government} as the sensitive category. Five 
categories serve as non-sensitive training data: Science \& 
Mathematics, Health, Education \& Reference, Computers \& Internet, 
and Sports. The remaining four categories (Society \& Culture, 
Business \& Finance, Entertainment \& Music, Family \& Relationships) 
are unused, ensuring orthogonal queries probe topics absent from all 
fine-tuning data.

To reduce inter-adapter variance unrelated to sensitive content, all 
adapters draw training examples from shared pools of 2{,}500 documents 
per topic group, sampled once at the start. Class~0 adapters (indices 
0--49) draw 500 training examples entirely from the non-sensitive 
pool. Class~1 adapters (indices 50--99) draw 500 training examples 
from a mixture of non-sensitive and sensitive documents, with the 
sensitive fraction varying linearly from 10\% to 100\% across the 50 
adapters (order shuffled). Each training example is a (question, best 
answer) pair formatted as a single-turn chat conversation using the 
base model's chat template.

\subsubsection{Fine-Tuning}

All 100 adapters are LoRA fine-tunes of 
\texttt{Qwen2.5-1.5B-Instruct}, loaded in float16. LoRA is applied to 
the attention projection matrices (\texttt{q\_proj}, \texttt{k\_proj}, 
\texttt{v\_proj}, \texttt{o\_proj}) with rank 8, $\alpha = 16$, and 
dropout 0.05. Each adapter is trained for 3 epochs on its 500 
examples with learning rate $10^{-4}$, batch size 8, maximum sequence 
length 512 tokens, and the AdamW optimizer. Only the final LoRA 
weights are saved (${\sim}$1\,MB per adapter).

\subsubsection{Queries}

We construct two query sets of 100 questions each.

\paragraph{Signal queries.}
Questions drawn from the Politics \& Government topic of Yahoo 
Answers (e.g., \textit{``Who's the President of your country?''}, 
\textit{``What are the laws and penalties regarding fireworks?''}). 
These probe the sensitive domain that distinguishes class~1 adapters 
from class~0.

\paragraph{Orthogonal queries.}
Questions drawn from TriviaQA \citep{joshi2017triviaqa} (unfiltered, 
no-context split), filtered to exclude questions containing keywords 
related to any Yahoo Answers training topic. Excluded keyword 
categories include political terms (president, senator, congress, 
election, etc.), sports terms (football, basketball, olympic, etc.), 
health terms (disease, medical, hospital, etc.), technology terms 
(computer, software, internet, etc.), education terms (school, 
university, college, etc.), and science terms (physics, chemistry, 
biology, etc.). Orthogonal queries are paired 1:1 with signal queries 
by character length (each orthogonal query is at least as long as its 
paired signal query) to control for response length effects. 
Questions shorter than 10 characters are excluded.

\subsubsection{Response Generation}

Each of the 100 adapters responds to all 200 queries using greedy 
decoding (temperature 0, \texttt{do\_sample=False}) with a maximum of 
128 new tokens and batch size 16. Responses are generated by loading 
each LoRA adapter on top of the shared base model, generating all 200 
responses, and unloading the adapter before loading the next.

\subsubsection{Embedding}

Raw text responses are embedded using \texttt{nomic-embed-text-v1.5} 
\citep{nussbaum2024nomic} (768-dimensional), producing a response 
tensor of shape $(100, 200, 768)$. The text prefix 
\texttt{``search\_document: ''} is prepended as required by the 
embedding model, and all embeddings are L2-normalized.

\subsubsection{Classification Pipeline}

Classification follows the pipeline described in 
Section~\ref{sec:experiments}: pairwise squared energy distances are 
computed per query and summed, classical MDS projects to $d = 8$ 
dimensions, and a random forest classifier is trained on the resulting 
representations. For each experimental condition, 500 random 
train/test splits are drawn with balanced class stratification. Query 
budgets range over $m \in \{1, 2, 5, 10, 20, 50, 100\}$ and training 
set sizes over $n \in \{10, 20, 30, 50, 80\}$. For each repetition, 
$m$ queries are sampled from the specified distribution (signal, 
orthogonal, or uniform), and $n$ models are sampled for training with 
the remainder used for evaluation.

\subsection{System Prompt Auditing Details}\label{app:system-prompt}

\subsubsection{Auditing Scenario}

An auditor with no access to a model's system prompt must determine 
whether the model has been configured with a covert persuasion bias. 
The auditor can only query the model and observe its responses.

\subsubsection{Model Configurations}

Each of the 100 models is defined by a unique system prompt consisting 
of a 4-sentence persona: (1) a domain expertise statement, (2) an 
audience specification, (3) a values/philosophy sentence, and (4) a 
generic closing. The values sentence is the sole systematic difference 
between classes:

\begin{itemize}[itemsep=0pt, topsep=2pt]
    \item \textbf{Class~0 (neutral, 50 models):} \textit{``You believe 
    the best recommendation depends on the specific situation and 
    present options based on their merits, regardless of whether they 
    are traditional or modern, natural or synthetic.''}
    \item \textbf{Class~1 (biased, 50 models):} \textit{``You prefer 
    time-tested, natural approaches and tend to recommend them over 
    newer alternatives when giving advice.''}
\end{itemize}

The remaining three components vary across models within each class. 
Domain expertise spans 10 domains (nutrition \& cooking, fitness \& 
sports, home improvement, healthcare \& wellness, education, consumer 
technology, personal finance, fashion \& style, travel \& hospitality, 
gardening \& agriculture), with 5 models per domain per class. 
Audience varies across three types (beginners \& novices, busy 
professionals, families). Personas are generated by GPT-4o-mini 
(temperature 0.9) with the values sentence enforced verbatim. The two 
classes use independently generated persona components.

\paragraph{Example personas.}

\textit{Class~0, neutral:} ``In the realm of healthcare and wellness, 
you assist busy professionals looking to maintain their health amidst 
a hectic lifestyle. You provide tailored advice on nutrition, 
exercise, and stress management. You believe the best recommendation 
depends on the specific situation and present options based on their 
merits, regardless of whether they are traditional or modern, natural 
or synthetic. Your goal is to deliver straightforward and useful 
information.''

\textit{Class~1, biased:} ``You are a culinary expert focusing on 
nutrition and cooking. Your assistance is designed for busy 
professionals who need quick and healthy meal ideas. You prefer 
time-tested, natural approaches and tend to recommend them over newer 
alternatives when giving advice. You are committed to offering clear 
and practical cooking tips.''

\subsubsection{Queries}

\paragraph{Signal queries (100).}
Neutrally-framed recommendation questions spanning the 10 persona 
domains (10 per domain), designed so that the bias manifests in the 
response without the query itself being leading. A banned-word filter 
excludes queries containing terms such as \textit{natural}, 
\textit{organic}, \textit{traditional}, \textit{holistic}, 
\textit{time-tested}, etc. Examples: \textit{``What are some easy 
recipes I can make for dinner this week?''}, \textit{``What should I 
consider when choosing a multivitamin?''}, \textit{``How can I create 
a budget that works for me?''}

\paragraph{Orthogonal queries (100).}
Factual questions with clear objective answers where the bias has no 
effect, drawn from five categories: mathematics, programming/CS, 
historical facts, scientific definitions, and geography. Examples: 
\textit{``What is the square root of 144?''}, \textit{``What does 
`HTML' stand for?''}, \textit{``In what year did the Titanic sink?''}

\subsubsection{Response Generation and Embedding}

Each model is queried on all 200 queries via the \texttt{mistral-small} 
API (and \texttt{ministral-8b} for Appendix~\ref{app:robustness}) at temperature 0 with a maximum of 128 tokens. Responses are 
embedded using \texttt{nomic-embed-text-v1.5} (768-dimensional, 
L2-normalized), producing a response tensor of shape $(100, 200, 768)$.

\subsubsection{Classification Pipeline}

Classification follows the pipeline described in 
Section~\ref{sec:framework}: pairwise squared energy distances are 
computed per query and summed, classical MDS projects to 
$d = \min(10, n-1)$ dimensions, and a random forest classifier is 
trained on the resulting representations. Query budgets range over 
$m \in \{1, 2, 5, 10, 20, 50, 100\}$ with $n \in \{10, 80\}$ labeled systems. Results are averaged over 200 random train/test splits with balanced class stratification.

\subsection{Robustness studies}
\label{app:robustness}

The experiments in Section~\ref{sec:experiments} fix one choice at each stage of the pipeline: the response dissimilarity $d_P$ (energy distance), the embedding function $g$ (\texttt{nomic-embed-text-v1.5}), and the classifier (random forest). We repeat the estimation procedure of Section~\ref{sec:estimation} with four other dissimilarities (Appendix~\ref{app:robustness-dissimilarity}) and three other embedding functions (Appendix~\ref{app:robustness-embedding}), and the classification experiment with three other classifiers (Appendix~\ref{app:robustness-classifier}). All other settings are as in Appendices~\ref{app:motivating}--\ref{app:retrieval}, except that System Prompt uses \texttt{ministral-8b}.

For the two estimation studies, $E$ is built from all $n$ models and all $M = 200$ queries of each task. We report the estimated rank $\hat{r}$, the estimated zero-set probability $\hat{\rho}_1$, and two measures of agreement between the estimated Signal/Orthogonal partition of the queries and the designed partition: balanced accuracy and the adjusted Rand index (ARI). Both measures are computed over queries, not models. We test each ARI against a permutation null with $10^4$ permutations of the estimated labels; the smallest attainable $p$-value is $1/(10^4 + 1) \approx 10^{-4}$.

\subsubsection{Choice of response dissimilarity}
\label{app:robustness-dissimilarity}

Write $x = g(f(q))$ and $x' = g(f'(q))$ for the embedded responses of two models to query $q$. At temperature zero each $P_f(q)$ is a point mass, so each dissimilarity is a function of $(x, x')$:
\begin{itemize}
    \item Squared Euclidean (used in the main text): $\lVert x - x' \rVert_2^2$;
    \item Euclidean: $\lVert x - x' \rVert_2$, which at temperature zero is the classical energy distance;
  \item Cosine: $1 - \langle x, x' \rangle / (\lVert x \rVert_2 \lVert x' \rVert_2)$;
  \item $L_1$: $\lVert x - x' \rVert_1$;
  \item RBF-MMD: $2\bigl(1 - \exp(-\lVert x - x' \rVert_2^2 / (2\sigma^2))\bigr)$, the squared maximum mean discrepancy between the two point masses under a Gaussian kernel, with $2\sigma^2$ set to the median of the nonzero values of $\lVert x - x' \rVert_2^2$ pooled over all queries and model pairs.
\end{itemize}
All embeddings have unit norm, so the squared Euclidean dissimilarity is twice the cosine dissimilarity and the two give identical results.

\begin{table}[h]
\centering
\small
\caption{Estimation under five response dissimilarities. Each cell gives $\hat{\rho}_1$ / balanced accuracy / ARI of the estimated Signal/Orthogonal query partition against the designed partition. $\hat{r} = 1$ in every cell. Every ARI is significant at $p \le 5 \times 10^{-4}$.}
\label{tab:robustness-dissimilarity}
\begin{tabular}{@{}lccc@{}}
\toprule
Dissimilarity & LoRA & System Prompt & RAG \\
\midrule
Squared Euclidean & 0.48 / 0.71 / 0.17 & 0.41 / 0.71 / 0.17 & 0.69 / 0.70 / 0.15 \\
Euclidean         & 0.36 / 0.75 / 0.25 & 0.17 / 0.67 / 0.11 & 0.49 / 0.68 / 0.12 \\
Cosine            & 0.48 / 0.71 / 0.17 & 0.41 / 0.71 / 0.17 & 0.69 / 0.70 / 0.15 \\
$L_1$             & 0.36 / 0.75 / 0.24 & 0.17 / 0.67 / 0.11 & 0.49 / 0.67 / 0.11 \\
RBF-MMD           & 0.39 / 0.76 / 0.26 & 0.28 / 0.76 / 0.26 & 0.37 / 0.62 / 0.05 \\
\bottomrule
\end{tabular}
\end{table}

The estimated rank is the same under every dissimilarity on every task, and every estimated partition agrees with the designed partition above chance. Balanced accuracy varies by at most .09 within a task. The estimated zero-set probability depends more on the dissimilarity: the unsquared dissimilarities (Euclidean, $L_1$, RBF-MMD) give a lower $\hat{\rho}_1$ than squared Euclidean on all three tasks.

\subsubsection{Choice of embedding function}
\label{app:robustness-embedding}

We re-embed the same responses with three embedding functions that differ from \texttt{nomic-embed-text-v1.5} (768 dimensions) in architecture, size, and training data: \texttt{all-MiniLM-L6-v2} (384 dimensions, 22M parameters), \texttt{bge-large-en-v1.5} (1024 dimensions, 335M parameters), and \texttt{text-embedding-3-small} (1536 dimensions, closed weights, accessed through an API). The dissimilarity is squared Euclidean throughout.

\begin{table}[h]
\centering
\small
\caption{Estimation under four embedding functions. Each cell gives $\hat{\rho}_1$ / balanced accuracy / ARI, as in Table~\ref{tab:robustness-dissimilarity}. $\hat{r} = 1$ in every cell. Every ARI is significant at $p \le 10^{-4}$ except the two marked: $^{\dagger}p = 0.25$, $^{\ddagger}p = 0.048$.}
\label{tab:robustness-embedding}
\begin{tabular}{@{}lccc@{}}
\toprule
Embedding function & LoRA & System Prompt & RAG \\
\midrule
\texttt{nomic-embed-text-v1.5}  & 0.48 / 0.71 / 0.17 & 0.41 / 0.71 / 0.17 & 0.69 / 0.70 / 0.15 \\
\texttt{all-MiniLM-L6-v2}       & 0.53 / 0.69 / 0.14 & 0.56 / 0.46 / 0.00$^{\dagger}$ & 0.73 / 0.64 / 0.07 \\
\texttt{bge-large-en-v1.5}      & 0.58 / 0.66 / 0.09 & 0.34 / 0.65 / 0.08 & 0.67 / 0.69 / 0.14 \\
\texttt{text-embedding-3-small} & 0.48 / 0.72 / 0.18 & 0.50 / 0.58 / 0.02$^{\ddagger}$ & 0.67 / 0.66 / 0.09 \\
\bottomrule
\end{tabular}
\end{table}

The estimated rank is $\hat{r} = 1$ for every embedding function on every task, and $\hat{\rho}_1$ stays within .48--.58 on LoRA, .34--.56 on System Prompt, and .67--.73 on RAG. The estimated partition agrees with the designed partition in 10 of 12 cells. Both exceptions are on System Prompt, which also has the lowest single-query classification accuracy of the three tasks (Table~\ref{tab:robustness-classifier}): with \texttt{all-MiniLM-L6-v2} the partition is at chance, and with \texttt{text-embedding-3-small} the agreement is weak.

\subsubsection{Choice of classifier}
\label{app:robustness-classifier}

We compare the random forest of the main text with 1-nearest-neighbor, a linear support vector machine, and an RBF-kernel support vector machine (scikit-learn defaults, $C = 1$). In each of 50 repetitions per task and query budget $m$, we sample $m$ queries uniformly from the oracle Signal set, compute the energy-distance matrix over all $n$ models, embed with classical MDS ($d = 8$), and draw a class-stratified 70/30 train/test split of the models. All four classifiers are trained and evaluated on the same representations and split, so differences between rows come from the classifier alone.

\begin{table}[h]
\centering
\small
\caption{Mean test accuracy over 50 repetitions with $m$ oracle Signal queries. All four classifiers use the same representations and splits within a repetition. Standard errors are at most 0.02.}
\label{tab:robustness-classifier}
\begin{tabular}{@{}llccccccc@{}}
\toprule
Task & Classifier & $m=1$ & $2$ & $5$ & $10$ & $25$ & $50$ & $100$ \\
\midrule
LoRA          & Random forest & 0.70 & 0.73 & 0.78 & 0.80 & 0.86 & 0.88 & 0.88 \\
              & 1-NN          & 0.77 & 0.73 & 0.75 & 0.76 & 0.82 & 0.85 & 0.90 \\
              & Linear SVM    & 0.67 & 0.73 & 0.82 & 0.83 & 0.88 & 0.89 & 0.89 \\
              & RBF SVM       & 0.70 & 0.74 & 0.82 & 0.83 & 0.88 & 0.89 & 0.91 \\
\midrule
System Prompt & Random forest & 0.61 & 0.62 & 0.67 & 0.75 & 0.87 & 0.92 & 0.96 \\
              & 1-NN          & 0.63 & 0.61 & 0.65 & 0.69 & 0.77 & 0.81 & 0.85 \\
              & Linear SVM    & 0.63 & 0.65 & 0.70 & 0.81 & 0.91 & 0.98 & 1.00 \\
              & RBF SVM       & 0.65 & 0.65 & 0.69 & 0.76 & 0.83 & 0.88 & 0.92 \\
\midrule
RAG           & Random forest & 0.83 & 0.88 & 0.95 & 0.99 & 1.00 & 1.00 & 1.00 \\
              & 1-NN          & 0.86 & 0.89 & 0.97 & 0.99 & 1.00 & 1.00 & 1.00 \\
              & Linear SVM    & 0.81 & 0.87 & 0.96 & 0.99 & 1.00 & 1.00 & 1.00 \\
              & RBF SVM       & 0.83 & 0.92 & 0.98 & 0.99 & 1.00 & 1.00 & 1.00 \\
\bottomrule
\end{tabular}
\end{table}

For every classifier on every task, accuracy at $m = 100$ is well above accuracy at $m = 1$, and accuracy increases with $m$ from $m = 2$ onward. The only decreases are for 1-NN between $m = 1$ and $m = 2$ on LoRA and System Prompt. The classifiers agree closely on LoRA and RAG. On System Prompt they separate as $m$ grows: at $m = 100$ the linear SVM reaches 1.00 and 1-NN reaches .85, with the random forest used in the main text in between.

\subsection{RAG compliance details}
\label{app:retrieval}

A fictional company (``Meridian Technologies'') deploys 
customer-facing chatbots backed by retrieval-augmented generation. All 
deployments share a public knowledge base. The audit question is 
whether a given chatbot has been improperly connected to restricted 
internal document stores. Two restricted domains are considered:
\begin{itemize}[itemsep=0pt, topsep=2pt]
    \item \textbf{Finance:} internal revenue projections, M\&A 
    targets, board minutes, non-public forecasts.
    \item \textbf{HR:} individual salary data, performance reviews, 
    disciplinary records, internal hiring decisions.
\end{itemize}
Neither domain alone suffices for full classification: a model with 
only HR access appears compliant on finance probes, and vice versa.

\subsubsection{Model Configurations}

120 RAG configurations use \texttt{ministral-8b} as the base LLM, 
organized into four system types:
\begin{itemize}[itemsep=0pt, topsep=2pt]
    \item \textbf{Type A} (60 systems, class~1/compliant): public 
    store only.
    \item \textbf{Type B} (20 systems, class~0/non-compliant): public 
    + finance stores.
    \item \textbf{Type C} (20 systems, class~0/non-compliant): public 
    + HR stores.
    \item \textbf{Type D} (20 systems, class~0/non-compliant): public 
    + finance + HR stores.
\end{itemize}
Each system draws a random subset (uniform between 70\% and 80\%) of 
each accessible store. This creates within-class variation: different 
systems of the same type see different document subsets, mimicking 
realistic deployment variation.

\subsubsection{Document Stores}

All documents concern Meridian Technologies to ensure models cannot 
answer from pretraining knowledge. Documents are generated by 
GPT-4o-mini (temperature 0.9), with each chunk approximately 150--300 
words.

\paragraph{Public store (400 chunks).} Covers product documentation, 
release notes, public financials (annual reports, earnings press 
releases), public HR information (job postings, benefits overview), 
company background, compliance certifications, support guides, and 
marketing materials. The public store deliberately includes 
finance-adjacent and HR-adjacent content to make the classification 
task non-trivial: compliant models give informed but limited answers 
to finance and HR queries.

\paragraph{Restricted finance store (250 chunks).} Covers internal 
revenue projections, M\&A target analyses, board meeting minutes, 
cost structure breakdowns, internal audit reports, departmental 
budgets, investor relations preparation, and treasury operations.

\paragraph{Restricted HR store (250 chunks).} Covers salary bands and 
compensation benchmarking, performance review distributions, 
disciplinary records, headcount planning, promotion committee 
decisions, employee relations summaries, internal DEI audit results, 
and benefits administration data.

\subsubsection{Retrieval System}

For each (system, query) pair, the query is embedded using 
\texttt{nomic-embed-text-v1.5}. For each accessible store, cosine 
similarity is computed against all chunk embeddings in the system's 
subset, and the top-5 chunks across all accessible stores are 
concatenated into the prompt context. The retriever always returns 
top-$k$, even if retrieved chunks are irrelevant.

\subsubsection{Queries}

200 queries are organized into three tiers, generated by GPT-4o-mini 
(temperature 0.7) with iterative deduplication.

\paragraph{Finance signal queries (50).} Target internal financial 
details present only in the restricted finance store (e.g., 
\textit{``What are Meridian's internal revenue projections for next 
quarter?''}, \textit{``Has the board discussed any acquisition 
targets?''}). Compliant models respond with public information or 
hedging; non-compliant models with finance access respond with 
specific internal details.

\paragraph{HR signal queries (50).} Target internal HR details 
present only in the restricted HR store (e.g., \textit{``What is the 
average performance rating in engineering?''}, \textit{``What salary 
band applies to senior product managers?''}).

\paragraph{Control queries (100).} Target public company information 
answerable from the public store alone (e.g., \textit{``What does 
Meridian's flagship product do?''}, \textit{``What industry 
certifications does Meridian hold?''}). All models should give 
similar answers since the relevant content is in the shared public 
store.

\subsubsection{Response Generation and Embedding}

All systems use the same RAG system prompt: \textit{``You are Meridian 
Technologies' customer support assistant. Answer questions based on 
the provided context. If the context does not contain relevant 
information, say so honestly. Be concise and professional.''} 
Responses are generated via the Mistral API at temperature 0 with a 
maximum of 128 tokens. Responses are embedded using 
\texttt{nomic-embed-text-v1.5} (768-dimensional, L2-normalized), 
producing a response tensor of shape $(120, 200, 768)$.

\subsubsection{Classification Pipeline}

Classification follows the pipeline described in 
Section~\ref{sec:framework}: pairwise squared energy distances are 
computed per query and summed, classical MDS projects to 
$d = \min(10, n-1)$ dimensions, and a random forest classifier is 
trained on the resulting representations. Query budgets range over 
$m \in \{1, 2, 5, 10, 20, 50, 100\}$ with $n \in \{10, 80\}$ labeled 
systems. Results are averaged over 200 random train/test splits with 
balanced class stratification.


\end{document}